\documentclass{article}
\usepackage{ijcai20}

\usepackage{times}
\usepackage{soul}
\usepackage{url}
\usepackage[hidelinks]{hyperref}
\usepackage[utf8]{inputenc}
\usepackage[small]{caption}
\usepackage{graphicx}
\usepackage{amsmath}
\usepackage{amsthm}
\usepackage{booktabs}
\usepackage{enumitem}
\usepackage{amsmath}
\usepackage{amsfonts}
\usepackage{amssymb}
\usepackage{tabularx}
\usepackage{subcaption}

\newcommand{\A}{\mathcal{A}}

\newcommand{\transpose}{{\mbox{\tiny T}}}
\usepackage{algorithm}
\usepackage{algorithm}
\usepackage{algpseudocode}

\newcommand{\cA}{{\mathcal{A}}}

\newcommand{\cD}{{\mathcal{D}}}
\newcommand{\cS}{{\mathcal{S}}}

\newcommand{\cF}{{\mathcal{F}}}

\newcommand{\bA}{\textbf{A}}
\newcommand{\bS}{\textbf{S}}

\newcommand{\bR}{\mathbb{R}}
\newcommand{\bbE}{\mathbb{E}}

\newif\ifnotes\notestrue
\def\htien#1{}

\newtheorem{assumption}{Assumption}

\newtheorem{theorem}{Theorem}

\title{Generalized Maximum Causal Entropy for Inverse Reinforcement Learning}

\iftrue
\author{
Tien Mai$^1$
\and
Kennard Chan$^2$\And
Patrick Jaillet$^3$
\affiliations
$^1$SIS, Singapore Management University\\
$^2$Nanyang Technological University\\
$^3$EECS, Massachusetts Institute of Technology
\emails
atmai@smu.edu.sg,
chan.yanting.kennard@dhs.sg,
jaillet@mit.edu
}
\fi

\begin{document}

\maketitle

\begin{abstract}
We consider the problem of learning from demonstrated trajectories with inverse reinforcement learning (IRL). Motivated by a limitation of the classical maximum entropy model in capturing  the structure of the network of states, we propose an IRL model  based on a generalized version of the causal entropy maximization problem, which allows us to generate a class of maximum entropy IRL models. Our generalized model has an advantage of being able to recover, in addition to a reward function, another expert's function that would (partially) capture the impact of the connecting structure of the states on experts' decisions. 
Empirical evaluation on a real-world dataset and a grid-world dataset shows that our generalized model outperforms the classical ones, in terms of recovering  reward functions and demonstrated trajectories. 
\end{abstract}

\section{Introduction}
We are interested in inverse reinforcement learning (IRL) \cite{russell1998learning,abbeel2004apprenticeship,Ziebart2008maximum}, which refers to the problem of learning and imitating experts' behavior by observing  their demonstrated trajectories of states and actions in some planning space. The experts are assumed to make actions by optimizing some accumulated rewards associated with states that they visit and the actions they make. The learner then aims at recovering such rewards to understand how decisions are made, and ultimately to imitate experts' behavior.
The rationale behind IRL is that
although a reward function might be a succinct and generalizable representation of an expert behavior, it is often difficult for the experts to provide their reward functions, as opposed to giving demonstrations.
So far, IRL has been successfully applied in a wide range of problems such as  predicting driver route choice behavior \cite{Ziebart2008maximum} or planning for quadruped robots \cite{ratliff2009learning}.

Maximum entropy IRL \cite{Ziebart2008maximum,ziebart2010_IRL_Causal} is a powerful probabilistic approach that has received a significant amount of attention over the decade. The main advantages of this IRL framework is that it allows the removal of ambiguity between demonstrations and the expert policy and to cast the reward learning as a maximum likelihood estimation problem. One of the interesting properties of the framework is that the action distribution can be interpreted as a solution to a causal entropy maximization problem under constraints on the empirical expectation of the  rewards, which also provides a worst-case prediction log-loss guarantee \cite{ziebart2010_IRL_Causal}.

In fact, through demonstrations, most of the IRL models will return a reward function associated with states and actions, but give no information about the effect of the connecting structure of the states on expert's decisions. In other words, the way states are connected to others may have a significant impact on  expert's policy, but, to the best of our knowledge, this is not captured thoroughly by the classical IRL models. For example, when travelling in a transportation network, an experienced taxi driver may not only consider travelling costs, but also take into consideration the correlation between possible paths. In the next sections, we will provide a simple example to illustrate this issue. This type of issue has been widely investigated by numerous econometrics studies \cite{train2009discrete}.

Motivated by the above issue of the classical IRL models, we propose a generalized IRL model based on the principle of maximum causal entropy.
More precisely, we propose a generalized version of the causal entropy function considered in \cite{ziebart2010_IRL_Causal} and show that solving the corresponding generalized causal entropy maximization problem will yield a formulation to infer  action probabilities for the reward learning problem. Our generalized model is more flexible and robust than the classical ones, in the sense that it allows to recover, in addition to an expert's reward function, a function that may partially capture the correlation between different trajectories. From a theoretical point of view, our generalized model is also consistent with the maximum causal entropy scheme, and also holds a worst-case prediction log-loss guarantee.  We provide experiments using a real-world taxi trajectories and a grid-world dataset. Our results show that the generalized model performs better than other classical IRL ones, in terms of recovering expert's reward functions and recovering demonstrated trajectories.

\noindent \textbf{{Related work.}} The concept of generalized entropy has been investigated by a number of econometrics studies to derive more general classes of demand models
\cite{ge:Fudenberg2015stochastic,ge:Fosgerau2016generalized,ge:Fosgerau2017discrete,ge:Fosgerau2019inverse}.
However,  it seems that we are the first to study generalized models in the context of maximum causal entropy, which involves dynamic decisions, and apply them to IRL. 

Our algorithm directly generalize the maximum causal entropy model proposed in  \cite{ziebart2010_IRL_Causal}, so it is closely related to IRL methods proposed by \cite{Ho2016GAN_IRL}; \cite{Fu2017Robust_IRL}; \cite{Finn2016Guided}; and \cite{Levine2011nonlinearIRL}. The generative adversarial imitation learning algorithm proposed by \cite{Ho2016GAN_IRL} is a powerful approach that allows to  learn directly from demonstration without recovering a reward function. Nevertheless, in many scenarios, a reward function returned from IRL might be useful to infer expert's intentions or to avoid re-optimizing a reward function in a new environment. \cite{Finn2016_connectionIRL} show a connection between generative adversarial networks (GANs) \cite{Goodfellow2014GAN}, maximum entropy IRL and energy-based models. They also propose the adversarial IRL framework that allows to learn a reward function based on the GANs idea. \cite{Fu2017Robust_IRL} develop an algorithm based on this adversarial IRL framework, which provide a way to recover a reward function that is ``robust'' in different dynamic settings. These GANs-based algorithms all rely on the maximum causal IRL framework \cite{ziebart2010_IRL_Causal}, thus can be adapted to use with our generalized IRL model. There are also some methods aiming at learning nonlinear reward functions through, e.g., boosting structured prediction \cite{Bagnell2007boosting}, deep neural networks \cite{Wulfmeier2015_IRL_NN}  or Gaussian processes \cite{Levine2011nonlinearIRL}, which might also benefit from our generalized IRL model.   

\vspace{-1.5em}
\section{Background}\label{sec:background}
An IRL model typically relies on a Markov Decision Process (MDP), which consists of states, actions and transition probabilities when making an action at any state. We first consider an MDP for an agent, defined as $(\cS,\cA,p,R,\gamma)$, where $\cS$ is a set of states $\cS = \{1,2,\ldots,|\cS|\}$, $\cA$ is a finite set of actions, $p:\cS\times \cA\times\cS \rightarrow [0,1]$ is a transition probability function, i.e., $p(s'|a,s)$ is the probability of moving to state $s'\in\cS$ from $s\in \cS$ by performing action  $a\in \A$, $R(a,s|\theta)$ is a reward function of parameters $\theta$ and a feature vector $\cF(a,s)$ associated with making decision $a\in \cA$ at state $s\in\cS$, and $\gamma$ is a discount factor. 


In this work we consider the case of finite time horizon and undiscounted  MDP. We first denote  $\bA,\bS$  as sequences of actions and states:  $\bA = \{a_0,\ldots,a_{T-1}\}$,  $\bS = \{s_0,\ldots,s_{T-1}\}$, where $a_t\in\cA, s_t\in \cS$ are the action and state at time $t\in\{0,\ldots,T-1\}$. The probability of $\bA$ causally conditioned on $\bS$ is defined as 
$P(\bA||\bS) = \prod_{t=0}^{T-1}P(a_t|s_t)$, 
and the causal entropy of $\bA$ conditional on $\bS$ is defined as \cite{Kramer1998directed,Permuter2008Causal} 
$
    H(\bA||\bS) =\mathbb{E}_{\bA,\bS}[\ln P(\bA||\bS)].
$
Then, we seek an action distribution that maximizes the following maximum causal entropy function under a constraint on the empirical expectation of the reward
 \cite{ziebart2010_IRL_Causal}.
\begin{align}
	\underset{P(a_t|s_t)}{\text{maximize}}\qquad &  H(\bA||\bS) &\nonumber \\
	 \text{subject to} \qquad & \bbE_{\bS,\bA}[R(\bS,\bA)] =  \widetilde{\bbE}_{\bS,\bA}[R(\bS,\bA)] & \nonumber\\
	  &  \sum_{a_t\in \cA^t}P(a_t|s_t) = 1,\ \forall s_t,  \nonumber
\end{align}
where $R(\bS,\bA)$ is the accumulated reward of actions $\bA$ and states $\bS$, which is a sum of action/state rewards as
$R(\bS,\bA) = \sum_{t=0}^{T-1}R(a_t,s_t)$, and 
$\cA^t$ is the set of possible actions at time $t$ and $\widetilde{\bbE}[R(\bS,\bA)]$ is the empirical expectation of the reward. \cite{ziebart2010_IRL_Causal} show that solving the above maximum entropy problem will yield a closed-form recursive formulation to infer the action distribution, making the training of the corresponding IRL model tractable.
We note that when the dynamics are deterministic, i.e., the transition probabilities only take values of 0 or 1, then the maximum causal entropy IRL proposed by \cite{ziebart2010_IRL_Causal} is identical to the maximum entropy IRL model introduced in their previous work \cite{Ziebart2008maximum}, which is also an energy-based model over all possible trajectories \cite{Finn2016_connectionIRL}.

\section{Generalized Maximum Causal Entropy}\label{sec:GMCE}
In this section, we will start by showing  a bottleneck of the classical maximum entropy models \cite{Ziebart2008maximum,ziebart2010_IRL_Causal}. We then propose our generalized maximum causal entropy (GMCE) IRL based on the maximum causal entropy principle.
We will also provide an algorithm  that can be used to practically train the GMCE IRL model.  Lastly, we  will take an example to show how the GMCE  gets over the limitation of the classical model.  

\subsection{Bottleneck of the Classical Models}
One of the issues of the maximum causal entropy (MCE) IRL model  is that it only relies on a reward function associated with states and actions and might not be able to capture the structure of the network of states, which would lead to an unreasonable probability distribution over trajectories.
To illustrate this, let us consider an IRL model with deterministic dynamics, i.e., the transition probabilities only take values of 0 or 1. In this context, it is well known that the probability of a trajectory $\sigma$ is
$
P(\sigma) = \exp(R(\sigma))\big/\left(\sum_{\sigma'\in\Omega} R(\sigma')\right)
$
where  $R(\sigma)$ is the accumulated reward of trajectory $\sigma$ and $\Omega$ is the set of all possible trajectories. This implies that for any $\sigma,\sigma'\in\Omega$, $P(\sigma)/P(\sigma') = \exp(R(\sigma))/\exp(R(\sigma'))$, which only depends on the rewards of the two trajectories. This property refers to a well-known issue in econometrics, called the Independence from Irrelevant Alternatives (IIA) property, which implies proportional substitution across trajectories. A number of econometrics studies argued that this property does not hold in general and should be relaxed \cite{train2009discrete,mcfadden1978modeling}.

We will take two simple examples (Fig. \ref{fig:Ex1} and \ref{fig:Ex2})  and bring some  insights from econometrics to further illustrate this issue.
In these examples, we assume that there are three paths going from an initial node (denoted by O) to a terminal node (D).
Links are the states of the model and we number them as in the figures. 
To make the examples simple, we assume that an action is defined as moving from a state to another state, which means that the MDP is deterministic. 
In the left example, there are three possible paths to go from O to D as $\{0,1,4\}, \{0,3,4\}$ and $\{0,2,4\}$ and in Fig. \ref{fig:Ex2} there are also three paths connecting O and D as $\{0,1,4,5\}, \{0,1,3,5\}$ and $\{0,2,5\}$.
We further assume that all the paths have the same rewards, i.e., $R(1)= R(2)=R(3)$ for Ex. 1 and $R(1)+R(3) = R(1)+R(4) = R(2)$ for Ex. 2. 
\vspace{-1em}
\begin{figure}[ht] 
  \begin{subfigure}[b]{0.5\linewidth}
    \centering
    \includegraphics[width=1.0\linewidth]{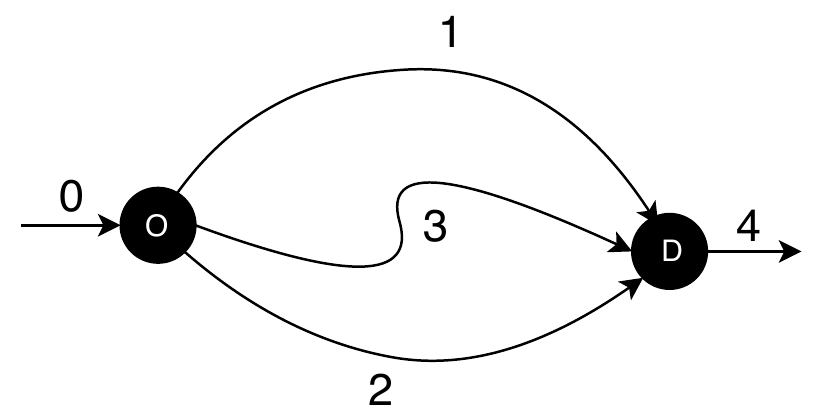} 
    \caption{Example 1} 
    \label{fig:Ex1} 
    \vspace{4ex}
  \end{subfigure}
  \begin{subfigure}[b]{0.5\linewidth}
    \centering
    \includegraphics[width=1.0\linewidth]{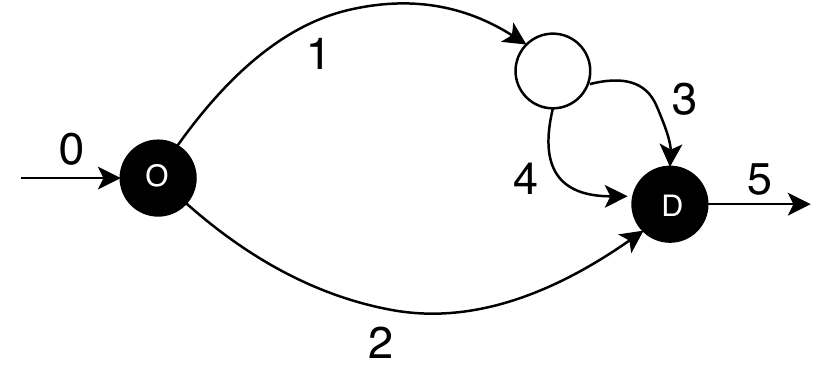} 
    \caption{Example 2} 
    \label{fig:Ex2} 
    \vspace{4ex}
  \end{subfigure} 
  \vspace{-3em}
    \caption{\small Simple examples to illustrate a limitation of the classical IRL model capturing the structure of the network of states.}
    \label{fig:3path-ex}
\end{figure}

Clearly, the MCE IRL model assigns the same probability of 1/3 to each of the three paths in each example. That makes sense for Ex. 1 where there is no overlap between the three paths. However, for Ex. 2, the MCE IRL model will still assign a probability of 1/3 to each of the three paths despite the overlap between two of the paths.  
More precisely, if the rewards of States 3 and 4 approach zero (but the three paths still have the same rewards), then from a behavioral point of view, one would not consider Path $\{0,1,3,5\}$ and $\{0,1,4,5\}$  as two distinct paths. Thus,
we expect that the probabilities of Path $\{0,2,5\}$ should approach 1/2 and the properties of Paths $\{0,1,3,5\}$ and $\{0,1,4,5\}$ should be close to 1/4.   
These probabilities cannot be archived by the standard IRL model as the three paths have the same rewards.
A more complicated example that shows how the correlation structure of  the states affects people's behaviour can be found in \cite{Mai2015_nestedRL}.

\subsection{Generalized Maximum Causal Entropy}
To deal with the aforementioned issue,  we generalize the MCE IRL model. To this end, let us define a generalized causal entropy function of actions $\bA$ conditional on a sequence of states $\bS$
\begin{equation}\label{eq:H-general}
    H^G(\bA||\bS)=\bbE_{\bA,\bS}\left[-\sum_{t=0}^T \ln G(P(a_t|s_t)|s_t)\right],
\end{equation}
where $G(p|s):\mathbb{R}_+ \times \cS \rightarrow \mathbb{R}_+$. Note in particular that while we will ultimately be interested in the case where $p\in [0,1]$, we extend the definition of $G$ to encompass all $p\in \mathbb{R}_+$. To derive closed-form recursive equations for the action probabilities, we require $G$ to satisfy  the following conditions (Assumption \ref{assum:as1}), for any $p\in \mathbb{R}_+$ and $s\in \cS$. 
\begin{assumption}\label{assum:as1}
Function $G(p|s):\mathbb{R}_+ \times \cS \rightarrow \mathbb{R}_+$ satisfies the following conditions
\begin{itemize}[align=left]
    \item [(i)] $G(p|s)$ and $\partial G(p|s)/\partial p$ both exist and are positive.
    \item[(ii)]$G(p|s)$ is invertible, i.e., there exists an unique function $G^{-1}(h|s):\mathbb{R}_+ \times \cS \rightarrow \mathbb{R}_+$ such that 
    $
    G^{-1}(G(p|s)|s) = p,\ \forall p\in\bR_+.
    $
    \item[(iii)] There exists a mapping $\mu:\cS \rightarrow \mathbb{R}_+$ such that 
    $
    {\partial \ln(G(p|s))}/{\partial p} = \mu(s)/p.
    $
\end{itemize}
\end{assumption}
Note that this implies there exists a mapping $\nu: \cS \rightarrow \mathbb{R}$ such  that $G(p|s)=e^{\nu(s)}p^{\mu(s)}$. Moreover, the above conditions also imply that function $G(p|\cdot)$ is multiplicative, i.e., $G(p_1p_2|s) = G(p_1|s)G(p_2|s)$, $\forall p_1,p_2\in\bR_+$. This property will be useful for deriving a closed-form solution for the generalized maximum causal entropy problem. 

Now, we aim at solving the following generalized causal entropy maximization problem under the generalized causal entropy function defined in Eq. \eqref{eq:H-general}. A solution to this problem will provide a way to infer a probability distribution over actions and states
\begin{align}
	\underset{P(a_t|s_t),\ \forall a_t,s_t}{\text{maximize}}\qquad &  H^G(\bA||\bS) &\label{prob:general-causal}\tag{P2}\\
	 \text{subject to} \qquad & \bbE_{\bS,\bA}[R(\bS,\bA)] =  \widetilde{\bbE}_{\bS,\bA}[R(\bS,\bA)] & \nonumber\\
	  &  \sum_{a_t\in \cA^t}P(a_t|s_t) = 1 &\ \forall t, s_t   \nonumber
\end{align}
The following theorem indicates that, under the conditions imposed on function $G(\cdot)$ in Assumption \ref{assum:as1}, there are closed forms to compute an optimal solution to (P2), making the training of the GMCE IRL model practically tractable. 
\begin{theorem}\label{thr:theorem1}
If function $G(p|s)$ satisfies Assumption \ref{assum:as1} and  $P(a_t|s_t)$, $\forall a_t\in\cA$, $s_t\in\cS$ is an optimal solution to the generalized maximum causal entropy problem Eq. \eqref{prob:general-causal}, then these probabilities can be computed by  the following recursive equations
{\small \begin{equation}\label{eq:general-entropy-recursion}
    \begin{aligned}
     Y_{a_t|s_t} &= \lambda R(s_t,a_t) + \mathbb{E}_{s_{t+1}}[\ln G(Z_{s_{t+1}}|s_{t+1})]\\
Z_{a_t|s_t} &= G^{-1} \left(e^{Y_{a_t|s_t}}|s_t\right),\ Z_{s_t} = \sum_{a_t \in\cA^t} Z_{a_t|s_t} \\
P(a_t|s_t) &= Z_{a_t|s_t}/Z_{s_t},\ \forall t, a_t,s_t, \\
    \end{aligned}
\end{equation}}
    where $\lambda$ is a scalar that depends on the empirical expectation of the reward $\widetilde{\bbE}_{\bS,\bA}[R(\bS,\bA)]$.
\end{theorem}
\begin{proof}(\textit{sketched}).
It suffices to find $P(a_t|s_t)$ that maximize $\cD = H^G(\bA||\bS)+\lambda \bbE_{\bS, \bA}[R(\bS,\bA)]$ for a constant $\lambda$. We denote
\[
\cD = \bbE_{\bS,\bA}\left[-\sum_{t=0}^T \ln G(P(a_t|s_t)|s_t)+\lambda R(\bS,\bA) \right].
\]
Using the method of Lagrange Multipliers, for each $s_t\in\cS$, we require that $\partial \cD/\partial P({a_t|s_t})$ are equal over all actions $a_t\in \cA^t$.  Taking the derivative of $\cD$ with respect to $P(a_t|s_t)$, removing parts that are equal over $a_t\in\cA^t$ and using Assumption \ref{assum:as1}, we  obtain
\begin{equation}\label{eq:th1-eq1}
P(a_t|s_t) \propto G^{-1}\left(\exp\left(\lambda R(s_t,a_t)+ U(s_t,a_t)\right)|s_t\right),
\end{equation}
where 
$U(s_t,a_t) = \sum_{k=t+1}^{T-1}\bbE_{s_k,a_k}\Big[\ln G(P(a_k|s_k)|s_k)-\lambda R(s_k,a_k)| a_t,s_t\Big]$. 
Now, by the multiplicativity of $G(p|\cdot)$ in $p$, we can reduce $U(s_t,a_t)$ as 
\begin{equation}\label{eq:th1-eq2}
U(s_t,a_t) = \sum_{s_{t+1}}P(s_{t+1}|a_t,s_t)\ln G(Z_{s_{t+1}}|s_{t+1}).    
\end{equation}
Combining Eq. \eqref{eq:th1-eq1} and Eq. \eqref{eq:th1-eq2}, we obtain the desired results.
\end{proof}

In fact, if $G(p|s) = p$, then the GMCE becomes the classical MCE IRL model. 
It is also possible to prove a worst-case prediction log-loss guarantee for the solution given in Theorem \ref{thr:theorem1}, as follows.
\begin{theorem}\label{thr:theorem2}
A solution to Eq. \eqref{prob:general-causal}  minimizes the following generalized worst-case prediction log-loss
\begin{equation}\label{eq:th2-eq0}
    \inf_{\substack{Q\in \Delta^T}} \sup_{\substack{P\in \Delta^T\\\widetilde{\bbE}^{P}(R) = \eta}}\bbE^P_{\bA,\bS}\left[-\sum_{t=0}^{T-1}\ln G(Q(a_t|s_t)|s_t)\right],
\end{equation}
where $\widetilde{\bbE}^{P}(R) = \bbE^P_{\bA,\bS}[R(\bA,\bS)]$ and $\eta = \widetilde{\bbE}_{\bS,\bA}[R(\bS,\bA)]$ (i.e., empirical expectation reward) and $\Delta^T = \{P(a_t|s_t), a_t\in \cA^t,\ s_t\in \cS, \sum_{a_t\in \cA^t}P(a_t|s_t) = 1\}$.
\end{theorem}
\begin{proof}(\textit{sketched}).
We can  follow the same strategy in \cite{Grunwald2004game} to prove the result. First, one can show that 
\begin{equation}\label{eq:th2-eq1}
\inf_{Q\in\Delta^T}\sum_{\bA, \bS} P(\bA,\bS)\left[-\sum_{t=0}^{T-1}\ln G(Q(a_t|s_t)|s_t)\right]    
\end{equation}
is only achieved uniquely at $Q=P$. This result can proved by taking the derivative of the Lagrangian function with respect to a variable $Q(a_t|s_t)$. Setting this derivative to zero  and using Assumption \ref{assum:as1} we can show that if $Q$ is a solution to Eq. \eqref{eq:th2-eq1}, then  $Q(a_t|s_t)$ needs to be equal to $P(a_t|s_t)$ for all $a_t,s_t$.

Thus, the generalized maximum causal entropy (P2) can be written as 
\begin{equation}\label{eq:th2-eq3}\small 
    \sup_{\substack{P\in \Delta^T\\ \widetilde{\bbE}^{P}(R) = \eta}} \inf_{\substack{Q\in \Delta^T}}\sum_{\bA,\bS}{P(\bA,\bS)}\left[-\sum_{t=0}^{T-1}\ln G(Q(a_t|s_t)|s_t)\right].
\end{equation}
Now, we see that the above objective function is convex in $Q$ (if we fix $P$) and concave in $P$ (if we fix $Q$), so using the \textit{Neumann's minimax theorem}, we can switch the sup-inf order and obtain an equivalent inf-sup problem, which is Eq. \eqref{eq:th2-eq0}. As we have seen, 
$Q=P$ is the unique solution to the infimum problem of Eq. \eqref{eq:th2-eq3} and $P$ that achieve the supremum of Eq. \eqref{eq:th2-eq3} is a solution to the maximum causal entropy problem (P2). This leads to our desired result.  
\end{proof}

Theorem \ref{thr:theorem2} says that the generalized maximum causal entropy can be viewed as a \textit{zero-sum game} where the opponent chooses a distribution over actions/states to maximize the predictor's generalized log-loss value, and the predictor tries to choose a distribution to minimize it. 

\subsection{Learning Algorithm}
We describe the main steps for computing the log-likelihood and its gradient in Algorithm \ref{algo:GenMaxCauEntIRL}. The algorithm performs a backward procedure from  $t=T$ to $t=0$. To make the algorithm general, we just assume that reward $R(a_t,s_t)$ is a function of  feature vector $\cF(a_t,s_t)$ and parameter $\theta$. 
Here, we also assume that $G(p|s)$ is a function of $p$, a feature vector  associated with  state $s$ and some parameters to be inferred through the training. That is, we can write $G(p|s) = G(p|s,\theta')$, where $\theta'$ is a vector of parameters of its own. The gradient $\nabla_\theta G^{-1}(e^{Y_{a_t|s_t}}|s_t)$ in Eq. \eqref{eq:algo-eq1} refers to the gradient vector of $G$ with respect to its own parameters $\theta'$. The gradient vector of the the log-likelihood can be straightforwardly derived from the recursive equations in Theorem \ref{thr:theorem1}.

\begin{algorithm}[htb]\small 
\caption{\textit{Log-likelihood and gradient computation}}
\label{algo:GenMaxCauEntIRL}
\begin{algorithmic}[1]
\For{\texttt{$t=T,...,1$}}
     \If{$t = T$}
 \State $\forall a_t,s_t$, set $Y_{a_t|s_t} = R(a_t,a_t), U_{a_t|s_t} = \nabla_\theta R(a_t,a_t)$ and $Z_{s_t} = 1$, $D_{s_t}=0$
\Else
$\ \forall a_t,s_t$
\[
\begin{aligned}
&Y_{a_t|s_t}\leftarrow R(s_t,a_t)
 +\sum_{s_{t+1}} p(s_{t+1}|a_t,s_t) \ln G(Z_{s_{t+1}}|s_{t+1})\\
& E_{s_t,s_{t+1}} \leftarrow \left.\frac{\partial G^{-1}(z|s_t)}{\partial z} \right\rvert_{{z = Z_{s_{t+1}}}}\frac{D_{s_{t+1}}}{G(Z_{s_{t+1}}|s_{t+1})}\\
& U_{a_t|s_t}\leftarrow \nabla_\theta R(s_t,a_t)+\sum_{s_{t+1}\in\cS} p(s_{t+1}|a_t,s_t)E_{s_t,s_{t+1}}
 \end{aligned}
 \]
\EndIf
\State For all $a_t\in \cA^t,s_t\in \cS$
\begin{align}
  Z_{a_t|s_t}&\leftarrow G^{-1}(e^{Y_{a_t|s_t}});\ 
  Z_{s_t}\leftarrow \sum_{a_t}Z_{a_t|s_t} s\nonumber \\
  D_{a_t|s_t}&\leftarrow \left.\frac{\partial G^{-1}(z|s_t)}{\partial z} \right\rvert_{{z = e^{Y_{a_t|s_t}}}}e^{Y_{a_t|s_t}} U_{a_t|s_t}\nonumber\\
  &\quad + \nabla_\theta G^{-1}(e^{Y_{a_t|s_t}}|s_t)\label{eq:algo-eq1} \\
  D_{s_t}&\leftarrow \sum_{a_t} D_{a_t|s_t}\nonumber
\end{align}
\EndFor
\State For any observation $(\widetilde{s}_t,\widetilde{a}_t)$ \Comment{LL and its gradients}
\begin{align}
    \ln P(\widetilde{a}_t|\widetilde{s}_t)& \leftarrow Z_{\widetilde{a}_t|\widetilde{s}_t}/Z_{\widetilde{s}_t},\nonumber \\ 
    \nabla_\theta\ln P(\widetilde{a}_t|\widetilde{s}_t) &\leftarrow {D_{\widetilde{a}_t|\widetilde{s}_t}}/{Z_{\widetilde{a}_t|\widetilde{s}_t}} - {D_{\widetilde{s}_t}}/{Z_{\widetilde{s}_t}}\nonumber
\end{align}
\end{algorithmic}
\end{algorithm}

As having said, the conditions in Assumption \ref{assum:as1} also implies that $G$ has the form $G(p|s) = e^{v(s)}p^{\mu(s)}$. One selection that would be of interest is $G(p|s) = p^{(\theta^{\mu})^\transpose \cF^{\mu}(s)}$, where $\theta^{\mu}$ is a vector of parameters to be inferred and $\cF^{\mu}(s)$ is a feature vector associated with state $s\in\cS$. To capture the structure of the network of states, $\cF^{\mu}(s)$ may contain some features representing the ``overlapping-level'' of the state, or sub-networks that $s$ belongs to. Such  features have been studied in the context of route choice modeling \cite{Fosgerau2013_RL}. 
We also denote $\theta^R$ as the parameter vector for the reward function and in a linear setting, we can write $R(a_t,s_t) = (\theta^R)^\transpose \cF(a_t,s_t)$.  The inverse of $G(p|s)$ becomes  $G^{-1}(p|s) = \exp\left( \ln p/((\theta^{\mu})^\transpose \cF^{\mu}(s))\right)$.

\subsection{Illustrative Example}
In the following we show how the GMCE IRL model gets over the aforementioned bottleneck of the classical models. 
We take the example in Fig. \ref{fig:Ex2} and keep the assumption that all the three paths from State 0 to State 5 have the same rewards. More precisely, we set $R(0) = R(5) = 0$, $R(2) = -5$, $R(1)=-4$ and $R(3)=R(4)=-1$. 
As being said, under the MCE IRL model, all the  three paths have the same probabilities of 1/3. Nevertheless,  it seems more reasonable to have a probability of more than 1/3 for Path $\{0,2,5\}$. We now show how this can be achieved by our GMCE model. 

\begin{figure}[ht] 
\centering
  \begin{subfigure}[b]{0.45\linewidth}
    \centering
    \includegraphics[width=1.0\linewidth]{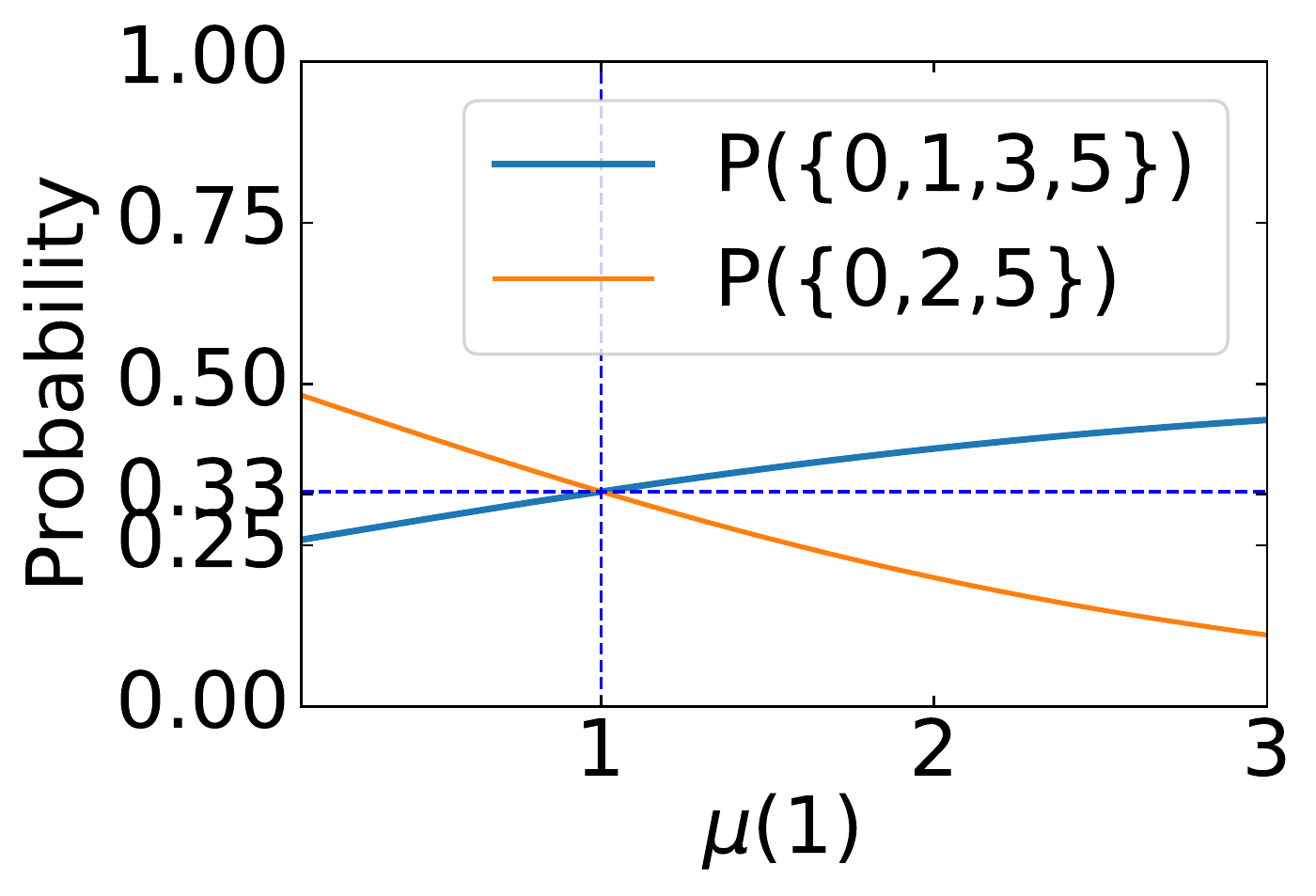} 
  \end{subfigure} 
  \begin{subfigure}[b]{0.45\linewidth}
    \centering
    \includegraphics[width=1.0\linewidth]{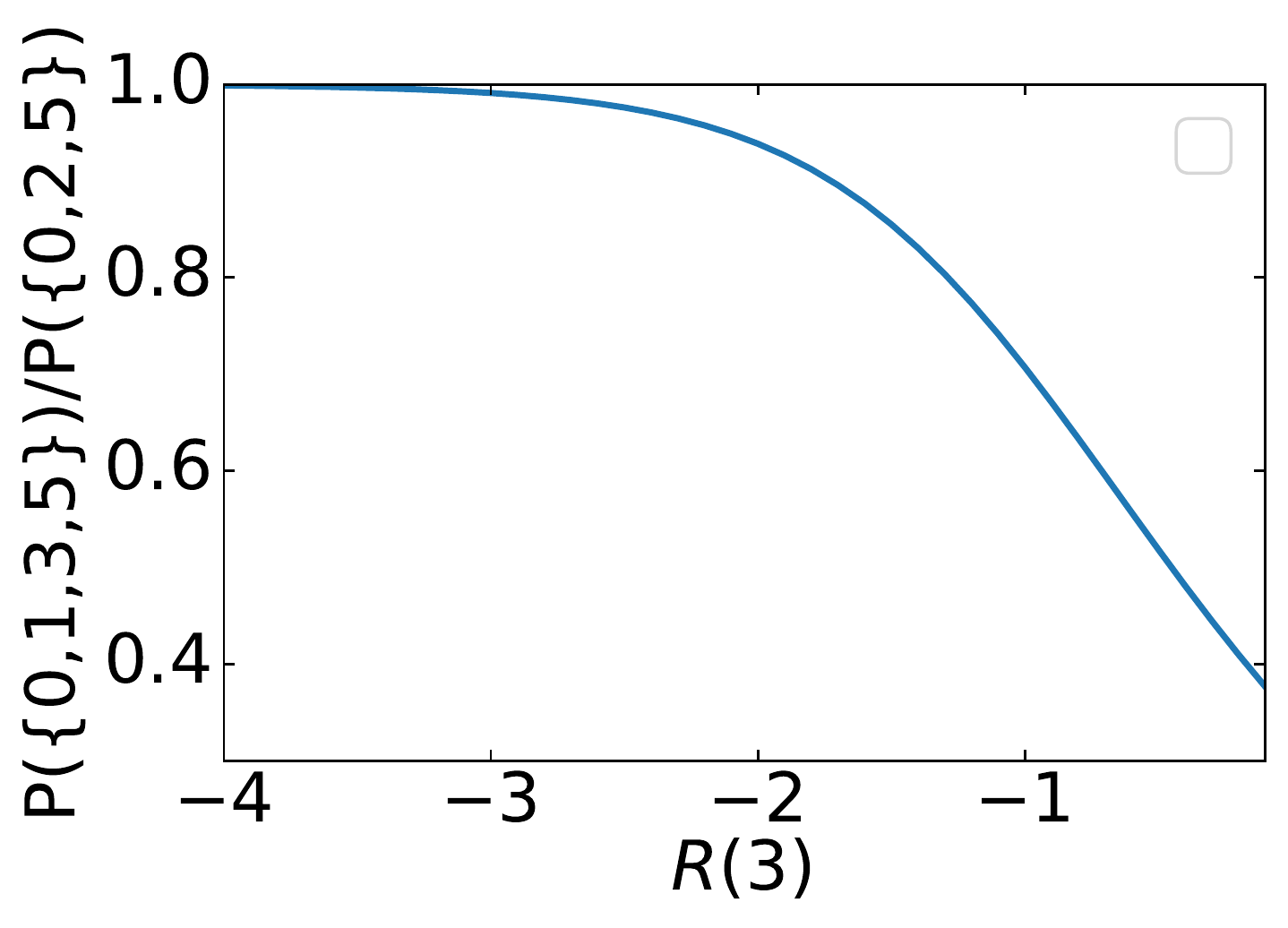} 
  \end{subfigure} 
    \caption{\small Example of probability distribution over three simple paths.
    }
    \label{fig:3path-ex-new-model}
\end{figure}

We take the GMCE IRL model specified by $G(p|s) = p^{\mu(s)}$. Since the two paths $\{0,1,3,5\}$ and $\{0,1,4,5\}$ all go though State 2 but  Path $\{0,2,5\}$ does not overlap with any other paths, we vary $\mu(1)$ while keeping $\mu(s)=1$ for all $s\neq 1$. The probabilities of the three paths with respect to different $\mu(1)$ are plotted in the left sub-figure of Fig. \ref{fig:3path-ex-new-model}, noting that $P(\{0,1,3,5\})$ always equals to $P(\{0,1,4,5\})$. Clearly, $P(\{0,2,5\}) > 1/3$ with  $\mu(1)<1$ and $P(\{0,2,5\}) > 1/3$ otherwise. When $\mu(1)$ goes 0, the probability of $\{0,2,5\}$ approaches 1/2 and the probabilities of $\{0,1,3,5\}$ and $\{0,1,4,5\}$ approach 1/4, which would be more reasonable path probabilities given the network structure.

To illustrate how the GMCE model relaxes the IIA issue  from the standard IRL model, we choose $\mu(1)=0.5$ and set $\mu(s)=1$ for all $s\neq 1$. 
The right sub-figure of Fig. \ref{fig:3path-ex-new-model} plots ratios between $P(\{0,1,3,5\})$ and $P(\{0,2,5\})$ with different $R(4)$ (i.e. reward of an irrelevant state). This ratio goes from 1 to 0.4 when we increase the $R(4)$ from -4 (low reward) to 0 (high reward), noting that if $\mu(1)=1$ (i.e. MCE model), then these ratios are always equal to 1 regardless the reward of State 4. These values are rational from a behavioral point of view, as when $R(4)$ is very low, we expect that Path $\{0,1,4,5\}$ would become unlikely to be chosen, and Paths $\{0,1,4,5\}$ and $\{0,2,5\}$ would be similar in terms of attractiveness. On the contrary, if $R(4)$ is very high, then due to the overlap between $\{0,1,4,5\}$ and $\{0,1,3,5\}$, Path $\{0,1,4,5\}$ is very unlikely to be chosen and we could expect that Path $\{0,1,3,5\}$ would be less attractive than Path  $\{0,2,5\}$, even though they have the same rewards.

In general, the value $\mu(s)$  would affect the attractiveness of all the trajectories passing through that state. So, by recovering such  values from demonstrated trajectories, we would expect to better learn how the structure of the network of states affects experts' behavior. Our experiments below  also show that the GMCE model performs better  than the classical IRL models in recovering expert's trajectories.

\section{Experiments}
This section evaluates the GMCE IRL model using two different datasets, namely, a real-world dataset which contains trajectories of taxi drivers (we shall refer to this dataset as the Transport dataset hereafter) and a simulated dataset obtained from a classical grid-world.
We use the GMCE model with linear-in-parameters reward functions. The model has two vectors of parameters to be inferred through the training, namely, $\theta^R$ for the action/state reward function and $\theta^{\mu}$ for $\mu(s)$. For the sake of comparison, we  will compare our generalized model with the MCE model  \cite{ziebart2010_IRL_Causal}. 
In the following, we first describe our datasets, and then show the comparison results. 
\subsection{Datasets and Experimental Settings}
We will evaluate the models based on log-likelihood and path-matching. We first split each dataset into a training set and a test set. The training set is first used to train the models. Next, for each trajectory in the test set, we feed its first state and last state to the models. Each model will then generate its most likely path based on the first and last states given. Matching is then performed between the particular trajectory from the test set and each of the most likely paths from the models. For each of the most likely paths, we count the number of states which also appeared in the test trajectory. Each count is then divided by the length of the test trajectory in order to give us a “percentage of matching”. The average matching metric is computed by taking the average of all the percentages of matching computed from every trajectory in the test set for a particular model. The 90\% Matching metric is the count of all percentages of matching that are greater or equal to 90\%, divided by the total number of percentages of matching.

\noindent \textbf{Transport dataset.}
The Transport dataset contains a total of 1832 trajectories of taxi drivers. The road network consists of 7288 links, which are regarded as states in our model. At each state, the set of available actions for a taxi driver is to move to one of the connected next states with no uncertainty. This means that the corresponding MDP is deterministic in nature. Four features are used to describe each of the states, namely,  left-turn, U-turn,  incident-constant, and travel time.
These features have  been used in some  established route choice modeling studies \cite{Fosgerau2013_RL,Mai2015_nestedRL}. Note that the application can be treated as a finite horizon problem, as it is rational to assume that a driver only considers paths that contain a finite number of links (i.e. states).  
We use the aforementioned four features to define the reward function $r(s|s)$. For $\mu(s)$, we use the number of incoming links and outgoing links at each state, and the Link-Size feature proposed by \cite{Fosgerau2013_RL} to capture the ``overlapping-level'' of states. 
 

\noindent \textbf{Grid-world dataset.}
The trajectories in the grid-world dataset is generated from a $5\times 5$ grid-world. The agent starts from the bottom-leftmost grid and has to move to the top-rightmost grid, which is also the only terminal state. 
The actions available in each grid are move left, move right, move up, move down, or stay in the same grid. Unlike the Transport dataset, the grid-world dataset is non-deterministic in nature, i.e.,  there is a 80\% chance that the agent will move in accordance with its intended action, and the remaining 20\% probability is distributed evenly to the remaining available actions. 

The actual rewards are given in the top-left sub-figure of Fig. \ref{fig:Rewards_grid-world}. Given the actual rewards, we apply Bellman's value iteration \cite{Bellman1957} to obtain the optimal policy for the grid-world. This optimal policy is then used in the grid-world to generate 200 trajectories. Out of these 200 trajectories, 160 trajectories are used to form the training set and the remaining 40 constitutes the test set. The MCE and the GMCE models are both trained using the training set and then evaluated using the test set. There are $5\times 5$ states and $5\times 5$ features and each feature corresponds to a position on the $5\times5$ grid, which is a state. For each state, the feature corresponding to it will take a value of 1, while the other features take zero values.  The features used for the definition of the reward function are the same as those used for the $\mu(s)$ values, for all $s\in \cS$.

\subsection{Comparison Results}
\textbf{Transport dataset.} We place 80\% of the taxi trajectories into the training set and the remainder into the test set. We report the comparison results in Table \ref{tab:table_of_results_for_transport}. The first and second rows clearly show that the GMCE model return significantly larger log-likelihood values for both training and test sets, as compared to the MCE model.

\begin{table}[htb]\small
\centering
\begin{tabular}{l|l|l}
  & MCE     & GMCE\\ 
\hline
Log Prob. (training)& -2074.3 & \textbf{-1988.8}   \\ 
\hline
Log Prob. (test)    & -566.4  & \textbf{-523.4 }   \\ 
\hline\hline
Avg. Matching & 87.6\%  & \textbf{89.3}\%    \\ 
\hline
90\% Matching & 63.5\%  & \textbf{67.1\%}    \\ 
\hline
Prob. of most likely path & 61.0\%  & \textbf{63.0\% }  
\end{tabular}
\caption{\small Comparison of the performance of the different IRL models using the Transport dataset.}
\label{tab:table_of_results_for_transport}
\end{table}

 For both measures on the third and fourth rows (Avg. Matching and 90\% Matching), we see that the GMCE model outperforms the MCE ones. 
 On the last row of Table \ref{tab:table_of_results_for_transport}, we provide the average probabilities of the most likely paths produced by each of the models. In other words, these values indicate how likely is the model going to produce the most likely path given a pair of origin and destination. The results show that, on average, the GMCE models  tend to assign higher probabilities to  their most likely trajectories, as compared to the classical one.

\textbf{Grid-world dataset.} Fig. \ref{fig:Rewards_grid-world}b shows the rewards recovered by the MCE model. The top-rightmost grid is correctly assigned a zero reward. Moreover, except for grid (4,3), the other grids are assigned somewhat significant negative rewards, which is consistent with the actual rewards of the grid-world. However, there are two main issues occurring with these recovered rewards. First, in the actual rewards of the grid-world as shown in Fig. \ref{fig:Rewards_grid-world}a, all grids except the top-rightmost one have the same reward value of -10. But this not captured by the recovered rewards of the MCE model. Second, a few grids, notably grid (4,3), are assigned reward values that are similar to the top-rightmost grid, which makes  MCE's rewards are quite different from the actual one.

\begin{figure}[ht] 
\centering
    \includegraphics[width=0.9\linewidth]{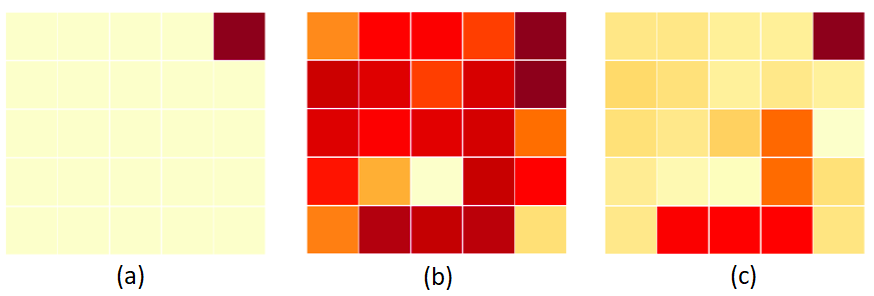} 
    \caption{\small Actual rewards and rewards recovered by the MCE and GMCE models. From left to right: Actual rewards, rewards by MCE, and rewards by the GMCE. } 
    \label{fig:Rewards_grid-world} 
\end{figure}

\vspace{-1em}
Fig. \ref{fig:Rewards_grid-world}c illustrates the rewards recovered by our GMCE model. Except for 5 grids near the bottom of the grid-world, all the other grids are assigned significantly negative rewards. These rewards are clearly better as compared to those from the MCE in two aspects. First, the GMCE model recovers a greater number of significantly negative grids. Second, the rewards assigned to all the grids excluding the top-rightmost grid are mostly uniform and quite similar in values.


We now move to  other comparing measures evaluating the ability of the models to recover demonstrated trajectories.    
The comparison results are reported in  Table \ref{tab:table_of_results_for_grid-world}. The first two rows  show the log-likelihood values attained by the two models (MCE, GMCE), which clearly indicate that the GMCE model  returns significantly larger log-likelihood values as compared to the classical MCE, on both training and test sets.
On the third and fourth rows, we see that the MCE  and   GMCE  perform equivalently in terms of Avg. Matching and 90\% Matching.
The reason is that both models have the same most likely path. This most likely path performs noticeably well as it matches most of the trajectories in the test set. However, the last row
shows that the GMCE  has a much higher chance of producing this most likely path as compared to the MCE model. 
In other words, the GMCE is more likely to produce a path that models the trajectories in the test set. In this sense, the GMCE  outperforms the MCE model. 

\vspace{-1em}
\begin{table}[ht]\small
\centering
\begin{tabular}{l|l|l}
  & MCE   & GMCE \\ 
\hline
Log Prob. (training)& -1984.2     & \textbf{-879.8}\textasciitilde{}   \\ 
\hline
Log Prob. (test)    & -533.1& \textbf{-241.9 } \\ 
\hline\hline
Avg. Matching & \textbf{87.8\% }  & \textbf{87.8\% } \\ 
\hline
90\% Matching & \textbf{67.5\% }  & \textbf{67.5\% } \\ 
\hline
Prob. of most likely path & 0.0\% & \textbf{8.0\% } 
\end{tabular}
\caption{\small Comparison of different IRL models using the grid-world dataset.}
\label{tab:table_of_results_for_grid-world}
\end{table}

\vspace{-2em}
\section{Conclusion}
In this work, we developed a generalized  IRL model that is consistent with the principle of the maximum causal entropy framework and holds a worst-case prediction log-loss guarantee. Our generalized model and algorithm have an advantage of being able to recover an additional expert's function that may capture the impact of the structure of the network on expert's policies. Our experiments clearly indicated the advantage of our generalized approach as compared to the classical ones. Many IRL models and applications would potentially benefit from our approach. In future work, we plan to develop generalized algorithms for IRL and imitation learning in the contexts of unknown or uncertain dynamics.  

\bibliographystyle{named}
\bibliography{refs}

\end{document}